\documentclass[12pt]{article}

\usepackage{geometry}
\usepackage{amsmath}
\usepackage{booktabs}
\usepackage{graphicx}
\usepackage{float}
\usepackage[linesnumbered,vlined,commentsnumbered,ruled]{algorithm2e}
 \usepackage{amsfonts}
 \newtheorem{theorem}{Theorem}
  \newtheorem{definition}{Definition}
   \newtheorem{proof}{Proof}
    \newtheorem{remark}{Remark}
\usepackage{authblk}
\begin{document}

\title{Nonparametric Risk Assessment and Density Estimation for Persistence Landscapes}

\author[1]{Soroush Pakniat \thanks{sorush.pakniat@atu.ac.ir}}
\author[2]{Farzad Eskandari \thanks{aeskandari@atu.ac.ir}}
\affil[1]{Faculty of Mathematics and Computer Science, Allameh Tabatabai  University}
\affil[2]{Faculty of Mathematics and Computer Science, Allameh Tabatabai  University}

\date{}
\maketitle

\begin{abstract}
This paper presents approximate confidence intervals for each function of parameters in a Banach space based on a bootstrap algorithm. We apply kernel density approach to estimate the persistence landscape. In addition, we evaluate the quality distribution function estimator of random variables using integrated mean square error (IMSE). The results of simulation studies show a significant improvement achieved by our approach compared to the standard version of confidence intervals algorithm. In the next step, we provide several
algorithms to solve our model. Finally, real data analysis shows that the accuracy of our method compared to that of previous works for computing the confidence interval.

\end{abstract}

\section{Introduction}
In recent years, the increased rate of data generation in some fields has emerged the need for some new approaches to extract knowledge from large data sets. One of the approaches for data analysis is topological data analysis (TDA), which refers to a set of methods for estimating topological structure in data (point cloud)(see the survey \cite{CarlssonTopologyData}; \cite{Ghrist}; \cite{CarlssonSurveyActa}; \cite{EdelsbrunnerAMSBook}). A persistence homology is a fundamental tool for extracting topological features in the nested sequence of subcomplexes (\cite{Edelsbrunner2002Survey}). In \cite{ChazalDataScientists}, the authors introduced a TDA from the perspective of data scientists. Since the use of TDA has been limited by combining machine learning and statistic subjects, we need to create a set of real-valued random variables that satisfy the usual central limit theorem and allow us to obtain approximate confidence interval and hypothesis testing. In the present study, we propose an alternative approach to approximate the sampling distribution and compute interval without some presupposition. This approach, which is asymptotically more accurate than the computation of standard intervals, analyzes a sample data population and identify the probability distribution of data. Some applications of TDA in various fields are summarized in the following: \newline 
A successful application of TDA was performed to extract the shape of breast cancer data in the form of the simplicial complex using Mapper technique by  \cite{NicolauCancer}. Computing the correlation of dynamic model of protein data and then this is input for topological methods by \cite{VioletaDynamicalProtein}, modeling the spaces to patches pixels and describing the global topological structure for patches \cite{CarlssonNaturalImage}, the use of computational topology for solving converage problem in sensor networks by \cite{SilvaSensorNetworks}, computing persistence homology for identifying the global structure of similarities between data by \cite{WagnerTowards}, applying persistence measures for the analysis of the observed spatial distribution of galaxies with Megaparsec scales by \cite{PranavCosmicWeb} are some potential applications of TDA.\cite{JMLR:v18:16-337} introduced a persistence image, which is the vectorization of persistence homology, and it applied on the dynamical system. \\
TDA has some fundamental aspects, \cite{BubenikCategorification} recreated a persistence homology based on a category theory and studied some features of $(\mathbb{R}, \leq)$, which consists of a set of objects and morphisms. Also \cite{BubenikGromovHausdorff} presented a generalization of Hausdorff distance, Gromov-Hausdorff distance, and the space of metric spaces in the form of categorical view. To generalize persistence module with the category theory and soft stability theorem see \cite{BubenikMetricGeneralized}. \cite{BubenikHigherInterpolation}, where the authors present a categorical language for construction embedding of a metric space into the metric space of persistence module.\\
In its standard paradigm, TDA computes the homology of point cloud that lies in some metric space. Thus, it creates a tool from algebraic topology such as simplicial complex, to eventually extract holes in topological space embedded in a $d$ - dimensional Euclidean space. It can be stated that there are $d$ different types of holes in dimensions $0$ to $d-1$. Moreover, there are additional topological attributes that we cannot distinguish between the feature of the original space and noise spawned in the process of changing the resolution. Thus, the persistence is one of the interest invariants in historical analysis. To compute persistence homology readers can refer to \cite{Zomorodian2005Survey} and  \cite{AfraBook}. \\
The space of persistence diagram is geometrically very complicated. In order to estimate Fr$\acute{e}$chet mean from the set of diagrams ($X_1 , \ldots , X_n$) by \cite{TurnerFrechet}, \cite{TurnerMeans} showed that the mean of the diagram is not unique but is unqiue for a special class of persistence diagram. Moreover, as can be seen, the space of persistence diagram is analogous to $L^p$ space. As a result, it is not plausible to use any parametric models for distribution. In this regard,  \cite{RobinsonArXiv} used randomization test where two set of diagrams are drawn from the same single distribution of diagrams. \cite{MileykoProbabilityMeasures} provides a theoretical basis for a statistical treatment that supports expectations, variance, percentiles, and conditional probabilities on persistence diagrams. \cite{MichelStatistical} introduces an alternative function on statistical analysis of the distance to measure (DTM) and estimates persistence diagram on metric space.  \cite{BlumbergStatistical} adapts persistence homology for computing confidence interval and hypothesis testing. Finally, \cite{ChazalStochastic} investigates the convergence of the average landscapes and bootstrap. \\
Due to the limitation of barcode and persistence diagram with combining statistics, we use a sequence of function such that $\lambda_k(t): \mathbb{R} \rightarrow \bar{\mathbb{R}}$ where $\bar{\mathbb{R}}$ denotes the extended real numbers and $\lambda_k(t)$ is persistence landscape (\cite{JMLR:v16:bubenik15a}). Next, we create a real-valued random variable by applying some functional in separable Banach space and we obtain the list of real-valued random variables.\newline
In the present work, we aimed at proposing a nonparametric inference of data to infer an unknown quantity to keep the number of underlying assumptions as weak as possible. Our approach would be of great assistance for the case that the modeler is unable to find a theoretical distribution that provides a good model for the input data. The main objective of this work is to present a generalized estimation of the confidence interval for large and small samples using a differentiable function of data and then nonparametric method to estimate probability density. As the first step, we estimate the CDF of random variables of persistence landscapes. To compute a large sample confidence interval, we use an empirical function that estimates the standard error of a statistical function of random variables. Next, we use bootstrap method for estimating the variance and distribution of random variables that we generate and random variate of the empirical distribution function and replace with main random variables. The goal of nonparametric density estimation of probability density is a few assumption about it as possible. Our estimator depends on a proper choice of smoothing parameter and kernel function that converges to the true density faster. We evaluate the quality of estimator with integrated mean squared error, followed by applying it to data sets of breast cancer.
 \begin{figure}
\centering
\includegraphics[scale=0.25]{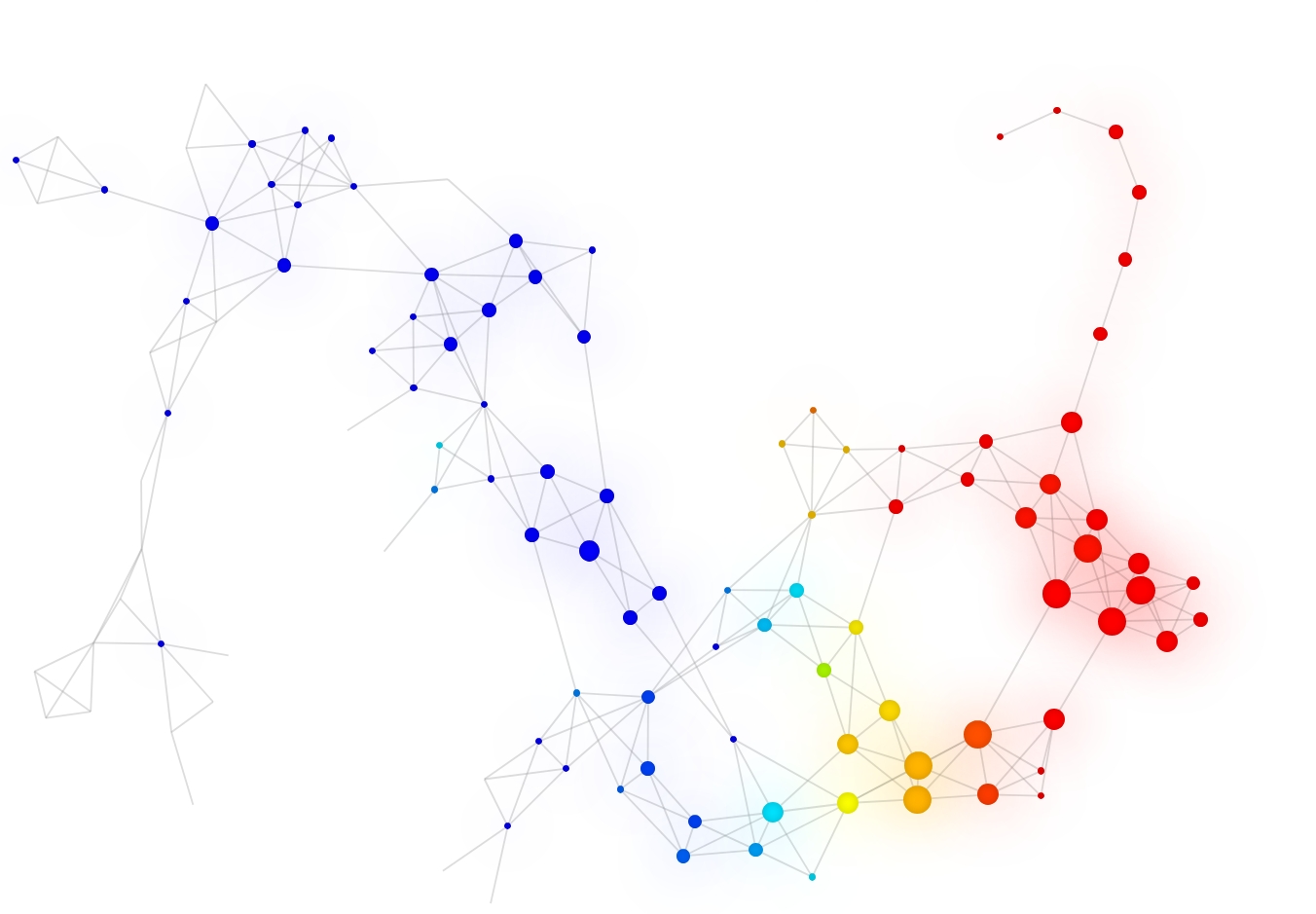}
\caption{The mapper methods applied to Wisconsin breast cancer dataset used by \cite{KeplerMapper} in python programming software package; the clusters correspond to covering of given point cloud.}
\label{fig:WisconsinBreastCancer}
\end{figure}
\newline
The remainder of this paper is organized as follows: In section \ref{section:Background}, we review the necessary background of persistence landscape. In section \ref{section:Nonparametric}, we provide theoretical background from nonparametric approach and algorithms. Finally, in section \ref{section:Applications}, we apply our approach on a sampling of objects and real datasets.
\section{Background of Persistence Landscape} \label{section:Background}
A simplicial complex $K$ is defined for representing a manifold and triangulation of topological space $X$. $K$ is a combinatorial object that is stored easily in computer memory and can be constructed by several methods in high dimensions with any metric space. A subcomplex $L$ of simplicial complex $K$ is a simplicial complex such that $L \subseteq K.$ A filtration of simplicial complex $K$ is a nested sequence of subcomplexs such that $K^0 \subseteq K^1 \subseteq \ldots \subseteq K^m.$ To create this object, you can see the (\cite{KhuyenGeneralizedCech}; \cite{ChambersVietorisRips}; \cite{DeyGraphInduced} and \cite{SilvaWitness}).\\
The fundamental group of space $X$ ($\pi_1(X,x_0)$ at the basepoint $x_0$), as an important functor in algebraic topology, consist of loops and deformations of loops. The fundamental group is one of the homotopy group $\pi_n(X,x_0)$ that has a higher differentiating power from space $X$, however, this invariant of topological space $X$  depends on smooth maps and is very complicated to compute in high dimensions. Thus, we must use an invariant of topological space that is computable on the simplicial complex. Homology groups show how cells of dimension $n$ attach to subcomplex of dimension $n-1$ or describe holes in the dimension of $n$ (connected components, loops, trapped volumes,etc.). The nth homology group is defined as $H_n= \ker \partial_n / \mbox{im} \partial_{n+1} = Z_n / B_n$ such that $\partial_n$ is the boundary homomorphism of subcomplexs, $Z_n$ is the cycle group and $B_n$ is boundary group. The nth Betti number $\beta_n$ of a simplicial complex $K$ is defined as $\beta_n = rank(Z_n) - rank(B_n)$. Through filtration step, we tend to extract invariant that remains fixed in this process, thus persistence homology satisfies this criterion for space-time analysis. Let $K^l$ be a filtration of simplicial complex $K$, the pth persistence of nth homology group of $K^l$ is $H_n^{b,d} = Z_n^b / (B_n^{b+d} \cap Z_n^b)$. The Betti number of pth persistence of nth homology group is defined as $\beta_n^{b,d}$ for the rank of free subgroup $(H_n^{b,d})$. To visualize persistence in space-time analysis, we should find the interval of $(i,j)$ that is invariant constantly through the filtration and obtain a topological summary from the point cloud.\\
Now, by rewriting the Betti number of the pth persistence of nth homology group, we have:
\begin{equation*}
\lambda (b,d) =
  \begin{cases}
    \beta^{b,d}     &  \mbox{if} \ b \leq d \\
    0   & otherwise
  \end{cases}
\end{equation*}
To convert $\lambda(b,d)$ function to a decreasing function, we change coordinate on it, Let $m = \dfrac{b+d}{2}$ and $h = \dfrac{d-b}{2}$. The rescaled rank function is:
\begin{equation*}
\lambda (m,h) =
  \begin{cases}
    \beta^{m-h,m+h}     &  \mbox{if} \ h \geq 0 \\
    0   & otherwise
  \end{cases}
\end{equation*}
\begin{definition}
The persistence landscape is a function  $\lambda: \mathbb{N} \times \mathbb{R} \rightarrow \bar{\mathbb{R}}$ where $\bar{\mathbb{R}}$ denoted the extended real numbers (introduced by \cite{JMLR:v16:bubenik15a}). In the other words, persistence landscape is sequence of function $\lambda_k : \mathbb{R} \rightarrow \bar{\mathbb{R}}$ such that:
\begin{equation}\label{eq:landscapeBubenik}
\lambda_k (t) = \sup (m \geq 0 | \beta^{t-m,t+m} \geq k).
\end{equation}
\end{definition}
We assume that our persistence landscape lies in separable Banach space ($L^p$). Let $Y:(\Omega,\mathcal{F},\mathcal{P}) \rightarrow \mathbb{R}$ be a real value random variable on underlying probability space, $\Omega$ is a sample space, $\mathcal{F}$ is a $\sigma$-algebra of events, and $\mathcal{P}$ is a probability measure. The expected value $E(Y) = \int Y dP$ and $\Lambda$ is the corresponding persistence landscape. If $f$ is a functional member of $L^q$ with $\dfrac{1}{p} + \dfrac{1}{q} = 1$, let
\begin{equation*}
Y = \int f \Lambda = \parallel f \Lambda \parallel_1
\end{equation*}
Then
\begin{equation*}
\sqrt{n} [\bar{Y_n} - E(Y)] \xrightarrow[]{\text{d}} N(0,Var(Y))
\end{equation*}
where $\parallel . \parallel_1$ denotes p-norm and $d$ denotes convergence in distribution. To computing confidence interval of real value random variable $Y$, we use the normal distribution to obtain the approximate ($1 - \alpha$) for $E(Y)$ as:
\begin{equation} \label{eq:bubenikConfidence}
\bar{Y_n} \pm z^* \dfrac{S_n}{\sqrt{n}}
\end{equation}
where $S^2_n = \dfrac{1}{n-1} \sum_{i=1}^n (Y_i - \bar{Y_n})^2$ and $z^*$ is the upper $\dfrac{\alpha}{2}$ critical value for the normal distribution. \newline
To apply persistence landscape on points, we choose a functional $f \in L^p$. If each $\Lambda(\Omega)$ is supported by $\lbrace 1, \ldots , K \rbrace \times [ -B,B ]$, take
\begin{equation} \label{functionalBubenik}
f(k,t) = \begin{cases}
1 & \mbox{if} \in [-B,B] \ \mbox{and} \ k \leq K \\
0 & \mbox{otherwise}
\end{cases}
\end{equation}
then $\parallel f \lambda \parallel_1 = \parallel \Lambda \parallel_1$.
\section{Nonparametric on Persistence Landscapes} \label{section:Nonparametric}
The basic idea of this approach is to use data to infer an unknown quantity without any presumption. For a more detailed exposition, we refer the reader to \cite{WassermanNonparametric}. The first problem is to estimate the cumulative distribution function (CDF), which is an important problem in our approach.
\begin{definition}
Let $X_1, \ldots , X_n \sim F$ where $F(x)=P(X \leq x)$. We estimate $F$ with the empirical distribution function $\widehat{F}_n$ which is the CDF that puts mass $\dfrac{1}{n}$ at each data point $X_i$. Formally,
\begin{equation*}
\widehat{F}_n = \dfrac{1}{n} \sum_{i=1}^n I(X_i \leq x)
\end{equation*}
where
\begin{equation*}
I(X_i \leq x) = 
\begin{cases}
1 & \mbox{if} \  X_i \leq x \\
0 & \mbox{otherwise}.
\end{cases}
\end{equation*}
\end{definition}
Let $X_1 , \ldots X_n \sim F$ and let $\widehat{F}_n$ be the empirical CDF, Then, at any fixed value of $x$
$E(\widehat{F}_n(x)) = F(x)$ and $V(\widehat{F}_n(x)) = \dfrac{F(x)(1 - F(x))}{n}$, 
where $V(\widehat{F}_n(x))$ denotes variance of empirical CDF.
\begin{definition}
A statistical functional $T(F)$ is any function of $F$. The plug-in estimator of $\theta = T(F)$ is defined by 
\begin{equation*}
\widehat{\theta}_n = T(\widehat{F}_n).
\end{equation*}
A functional of the form $\int a(x) dF(x)$ is called a linear functional where $a(x)$ denoted a function of $x$. The plug-in estimator for linear functional $T(F) = \int a(x)dF(x)$ is:
\begin{equation*}
T(\widehat{F}_n) = \int a(x) d \widehat{F}_n (x) = \dfrac{1}{n} \sum_{i=1}^n a(X_i).
\end{equation*}
\end{definition}
For an approximation of the standard error of a plug-in estimator, use the influence function as follows:
\begin{definition}
The G\^ateaux derivative of $T$ at $F$ in the direction $G$ is defined by:
\begin{equation*}
L_F (G) = \lim_{\epsilon \rightarrow 0} \dfrac{T((1 - \epsilon)F + \epsilon G)- T(F)}{\epsilon}
\end{equation*}
The empirical influence function is defined by $\widehat{L}(x) = L_{\widehat{F}_n} (x)$. Thus,
\begin{equation*}
\widehat{L}(x) = \lim_{\epsilon \rightarrow 0} \dfrac{T((1 - \epsilon)\widehat{F}_n + \epsilon G)- T(\widehat{F}_n)}{\epsilon} .
\end{equation*}
\end{definition}
\begin{theorem}\label{eq:influenceFunctions}
Let $T(F) = \int a(x) dF(x)$ be a linear functional. Then,
\begin{equation*}
L_F(x) = a(x) - T(F) \ \mbox{and} \ \widehat{L}(x) = a(x) - T(\widehat{F}_n),
\end{equation*}
Let
\begin{equation*}
\widehat{\tau}^2 = \dfrac{1}{n} \sum_{i=1}^n \widehat{L}^2 (X_i) = \dfrac{1}{n} \sum_{i=1}^n (a(X_i)-T(\widehat{F}_n))^2
\end{equation*}
then
\begin{equation*}
\widehat{\tau}^2 \xrightarrow[]{\text{P}} \tau^2 \ \mbox{and} \ \dfrac{\widehat{se}}{se} \xrightarrow[]{\text{P}} 1 \  \mbox{where} \ \xrightarrow[]{\text{P}} \ \mbox{denoted convergence in probability} \ \widehat{se} = \dfrac{\widehat{\tau}}{\sqrt{n}} \ \mbox{and} \ se = \sqrt{V(T(\widehat{F}_n))}.
\end{equation*}
\begin{proof}
We see $E(L_F(x)) = E(a(x)) - T(F) = T(F) - T(F)=0$. So, by the weak low of a large number(WLLN), it can easily be shown that $\hat{\tau}^2$ is a consistant  estimator for $\tau^2 = var(L_F(x)) = E(L_F(x)^2)$.
\end{proof}
\end{theorem}
\begin{definition}
If $T$ is Hadamard differentiable with respect to $d(F,G) = \sup_x |F(x) - G(x)|$ then
\begin{equation*}
\sqrt{n} (T(\widehat{F}_n)-T(F)) \leadsto N(0,\tau^2)
\end{equation*}
where $\tau^2 = \int L_F(x)^2 dF(x)$ and $\leadsto$ denotes convergence in distribution. Also,
\begin{equation*}
\dfrac{(T(\widehat{F}_n)-T(F))}{\widehat{se}} \leadsto N(0,1)
\end{equation*}
Such that
\begin{equation*}
\widehat{se} = \dfrac{\widehat{\tau}}{\sqrt{n}} \ \mbox{and} \ \widehat{\tau} = \dfrac{1}{n} \sum_{i=1}^n L^2(X_i).
\end{equation*}
\end{definition}
\subsection*{Bootstrap  Variance Estimation}
The nonparametric delta method is an approximation of $\dfrac{(T(\widehat{F}_n)-T(F))}{\widehat{se}} \leadsto N(0,1)$, A large sample confidence interval is $T(\widehat{F}_n) \pm z_{\alpha/2} \widehat{se}$. \\
The bootstrap is a method for estimating the variance and the distribution of a statistic $T_n = g(X_1 , \ldots , X_n)$. We can also use the bootstrap to construct confidence intervals, also the bootstrap estimate $V_F(T_n)$ with $V_{\widehat{F}_n(T_n)}$.
We estimation variance of $T_n$ with nonparametric bootstrap as follows:
\begin{enumerate}
\item Draw $X_1^* , \ldots , X_n^* \sim \widehat{F}_n$.
\item Compute $T_n^* = g(X_1^* , \ldots , X_n^*)$.
\item Repeat steps 1 and 2, $B$ times to get $T_{n,1}^* , \ldots , T_{n,B}^*$.
\item Let 
\begin{equation*}
v_{boot} = \dfrac{1}{B} \sum_{i=1}^{B} \Big( T_{n,b}^* - \dfrac{1}{B} \sum_{r=1}^B T_{n,r}^* \Big)^2.
\end{equation*}
\end{enumerate}
\begin{align*}
\begin{array}[t]{l} \mbox{Real world:}  \ F  \  \Longrightarrow X_1 , \ldots , X_n \Longrightarrow T_n = g(X_1 , \ldots , X_n) \\
\mbox{Bootstrap world:} \ \widehat{F}_n \Longrightarrow X_1^*,\ldots,X_n^* \Longrightarrow T_n^* = g(X_1^* , \ldots , X_n^*) \end{array}
\end{align*}
Also $V_F(T_n) \approx V_{\widehat{F}_n}(T_n) \approx v_{boot}$
\subsection{Bootstrap Confidence Intervals}
There are several ways to construct bootstrap confidence intervals that are difference from accuracy criterion.
\begin{itemize}
\item The simplest is the Normal interval, which is defined as,
\begin{equation*}
T_n \pm z_{\alpha / 2} \widehat{se}_{boot}
\end{equation*}
\item Let $\theta = T(F)$ and $\widehat{\theta}_n = T(\widehat{F}_n)$ be an estimator for $\theta$. We tend to estimate a nonparametric confidence interval for functions of $\theta$. The pivot $R_n = \widehat{\theta}_n - \theta$. Let $H(r)$ denotes the CDF of the pivot:
\begin{equation*}
H(r) = P_F(R_n \leq r).
\end{equation*}
Let $C_n^\ast = (a,b)$ where
\begin{equation*}
a = \widehat{\theta}_n - H^{-1} \big( 1 - \dfrac{\alpha}{2} \big) \ \mbox{and} \ b = \widehat{\theta}_n - H^{-1} \big( \dfrac{\alpha}{2} \big)
\end{equation*}
Since $a$ and $b$ depend on the unknown distribution $H$, we should form a bootstrap estimate of $H$ as:
\begin{equation*}
\widehat{H}(r) = \dfrac{1}{B} \sum_{i=1}^B I(R_{n,b}^* \leq r)
\end{equation*}
Where $R_{n,b}^* = \widehat{\theta}_{n,b}^* - \widehat{\theta}_n$. Let $r_\beta^*$ denote the $\beta$ sample quantile of $(R_{n,1}^* , \ldots , R_{n,B}^*)$ and let $\theta_\beta^*$ denote the $\beta$ sample quantile of $(\theta_{n,1}^* , \ldots , \theta_{n,B}^*)$. Note that $r_\beta^* = \theta_\beta^* - \widehat{\theta_n}$. Follows that an approximate $1 - \alpha$ confidence interval is $C_n = (\widehat{a},\widehat{b})$ is a nonparametric confidence interval a least $(1 - \alpha)$, where
\begin{align*}
\begin{array}[t]{l} \widehat{a} = \widehat{\theta}_n - \widehat{H}^{-1} \big( 1 - \dfrac{\alpha}{2} \big) = \widehat{\theta}_n - r_{1- \alpha/2}^* = 2 \widehat{\theta}_n - \theta_{1 - \alpha/2}^* \\
\widehat{b} = \widehat{\theta}_n - \widehat{H}^{-1} \big(\dfrac{\alpha}{2} \big) = \widehat{\theta}_n - r_{\alpha/2}^* = 2 \widehat{\theta}_n - \theta_{\alpha / 2}^* .\end{array}
\end{align*}
\item The $1 - \alpha$ bootstrap studentized pivotal interval is
\begin{equation*}
\big( T_n - z_{1 - \alpha/2}^* \widehat{se}_{boot} , T_n - z_{\alpha / 2}^* \widehat{se}_{boot} \big)
\end{equation*}
where $z_\beta^*$ is the $\beta$ quantile of $Z_{n,1}^*,\ldots , Z_{n,B}^*$ and
\begin{equation*}
Z_{n,b}^* = \dfrac{T_{n,b}^* - T_n}{\widehat{se}_b^*} .
\end{equation*}
\item The other approach for estimating the confidence interval for $h(\theta)$ is
\begin{equation*}
C_n = \big( T_{(B \alpha/2)}^* , T_{(B(1-\alpha)/2)} \big) ,
\end{equation*}
where $C_n$ is the bootstrap percentile interval in this approach, Just use the $\alpha / 2$ and $1 - \alpha/2$ quantiles of the bootstrap sample.
\end{itemize}
\subsection{Quality of Estimator}
The goal of nonparametric  density estimation is to estimate $f$ with as few assumptions about $f$ as possible. We denote the estimator by $\widehat{f}_n$. We will evaluate the quality of an estimator $\widehat{f}_n$ with the risk, or integrated mean squared error, $R=\mathbb{E}(L)$ where
\begin{align}
L = \int (\widehat{f}_n(x) - f(x))^2 dx
\end{align}
is the integrated squared error loss function. The estimators depend on some smoothing parameter $h$ chosen by minimizing an estimate of the risk. The loss function, which we now refer to as function of $h$, is:
\begin{align*}
L  \begin{array}[t]{l} = {\displaystyle \int} (\widehat{f}_n(x) - f(x))^2 dx \\
= {\displaystyle \int} \widehat{f}_n^2(x) dx - 2 {\displaystyle \int} \widehat{f}_n(x) f(x) dx + {\displaystyle \int} f^2(x) dx.\end{array}
\end{align*}
The last term does not depend on $h$ so minimizing the loss is equivalent to minimizing the expected value, therefore the cross-validation estimator of risk is:
\begin{equation} \label{eq:crossValidation}
\widehat{J}(h) = \int \big( \widehat{f}_n(x) \big)^2 dx - \dfrac{2}{n} \sum_{i=1}^n \widehat{f}_{(-i)} (X_i)
\end{equation}
where $\widehat{f}_{(-i)}$ is the density estimator obtained after removing the $i^{th}$ observation. 
\begin{theorem} \label{eq:theorem612}
Suppose that $f^\prime$ is absolutely continuous and that $\int \big( f^\prime (u) \big)^2 du < \infty$, Then,
\begin{align}
R(\widehat{f}_n , f ) = \dfrac{h^2}{12} \int \big( f^\prime (u) \big)^2 du + \dfrac{1}{nh} + o(h^2) + o (\dfrac{1}{n}).
\end{align}
Where $x_n = o(a_n)$ this means that $\lim_{n \rightarrow \infty} x_n / a_n = 0$.  The value $h^\ast$ that minimizes (\ref{eq:theorem612}) is
\begin{align}
h^\ast = \dfrac{1}{n^{1/3}} \Bigg( \dfrac{6}{\int ( f^\prime (u) )^2 du } \Bigg)^{1/3}.
\end{align}
With this choice of binwidth,
\begin{align}
R(\widehat{f}_n , f) \sim \dfrac{C}{n^{2/3}}
\end{align}
where $C = (3/4)^{2/3} \Big( \displaystyle\int \big( f^\prime (u) \big)^2 du \Big)^{1/3}$.
\end{theorem}
The proof of Theorem \ref{eq:theorem612} can be seen in appendix \ref{proof:611}. We see that with an optimally chosen binwidth, the risk decreses to $0$ at rate $n^{-2/3}$. Moreover, it can be seen that kernel estimators converge at the faster rate  $n^{-4/5}$ and that in a certain sense no faster rate is possible.\\
We discuss kernel density estimators, which are smoother and can converge to the true density faster. Here, the word kernel refers to any smooth function $K$ such that $K(x) \geq 0$ and
\begin{equation}
\int K(x) dx = 1, \ \int xK(x) dx = 0 \ \mbox{and} \ \sigma_K^2 \equiv \int x^2 K(x) dx > 0.
\end{equation}
Some commonly used kernels are the following:
\renewcommand{\arraystretch}{2.0}
\begin{table}[!h]
\begin{center}
\begin{tabular}{ || p{5cm} | p{5cm} ||  }
\hline
  the Gaussian kernel: & $K(x) = \dfrac{1}{\sqrt{2 \pi}} \exp^{-x^{2}/2}$  \\
  \hline
  the tricube kernel: & $K(x) = \dfrac{70}{81} \Big( 1- |x|^3 \Big)^3 I(x)$ \\
  \hline
\end{tabular}
\end{center}
\end{table}
where 
\begin{align*}
I(x) = \begin{cases}
1 & \mbox{if} \ |x| \ \leq 1 \\
0 & \mbox{otherwise}
\end{cases}
\end{align*}
\begin{definition} \label{def:densityEstimator}
Given a kernel $K$ and a positive number $h$, called the bandwidth, the kernel density estimator is defined to be
\begin{equation}
\widehat{f}_n(x) = \dfrac{1}{n} \sum_{i=1}^n \dfrac{1}{n} K \Big( \dfrac{x - X_i}{h} \Big).
\end{equation}
\end{definition}
\begin{theorem}
Assume that $f$ is continuous at $x$, $h_n \rightarrow 0$, and $n h_n \rightarrow \infty$ as $n \rightarrow \infty$. Then, by weak low of large number(WLLN), $\widehat{f}_n(x) \rightarrow f(x)$.
\end{theorem}
\begin{proof}
Please see \cite{WassermanNonparametric}
\end{proof}
\begin{remark}\label{theorem:628}
Let us now consider what happens when $f^\prime = 0$ but $f^{''} \neq 0$. Since the leading term in the Theorem \ref{eq:theorem612} drops out, we can carry Theorem \ref{eq:theorem612} one step further. \newline
Let $R_x = \mathbb{E} \Big( f(x) - \widehat{f}(x) \Big)^2$ be the risk at a point $x$ and $R = \int R_x dx$ denotes the integrated risk. Assume that $f^{''}$ is absolutely continuous and that $\int \Big( f^{'''} (x) \Big)^2 dx < \infty$. Then,
\begin{equation}
R_x = \dfrac{1}{4} \sigma_K^4 h_n^4 \Big( f^{''} (x) \Big)^2 + \dfrac{f(x) \int K^2(x) dx}{nh_n} + O \Big( \dfrac{1}{n} \Big) + O(h_n^6)
\end{equation}
and
\begin{equation}\label{eq:theorem629}
R = \dfrac{1}{4} \sigma_K^4 h_n^4 \int \Big( f^{''} (x) \Big)^2 dx + \dfrac{\int K^2 (x) dx}{nh} + O \Big( \dfrac{1}{n} \Big) + O (h_n^6)
\end{equation}
where $\sigma_K^2 = \int x^2 K(x) dx$ and $x_n = O(a_n)$ means that $|x_n / a_n|$ is bounded for all large $n$.
\end{remark}
The proof of Theorem \ref{theorem:628} is supplied in Appendix \ref{proof:628}.
Differentiate (\ref{eq:theorem629}) with respect to $h$ and set it equal to $0$ gives an asymptotically optimal bandwidth as:
\begin{equation}\label{eq:630}
h_\ast = \Big( \dfrac{c_2}{c_1^2 A(f)n} \Big)^{1/5}
\end{equation}
where $c_1 = \int x^2 K(x) dx$, $c_2 = \int K(x)^2 dx$ and $A(f) = \int \Big( f^{''}(x) \Big)^2 dx$, which explain that the best bandwidth decreases at a rate $n^{-1/5}$. \\
We compute $h_\ast$ from (\ref{eq:630}) under the idealized assumption that $f$ is normal. This choice of $h_\ast$, which is called the normal reference rule, works well if the true density is very smooth.
\subsection{Algorithms}
In this section, we represent our algorithm to compute confidence interval by the small and large sample and density estimation for random variables of persistence landscape with the nonparametric approach.
\subsubsection{Bootstrap Persistence Landscape}
Let us have a random sample $Y=[y_1, \ldots,y_n]$ from a cumulative distribution of $F$ and work on a variety estimation problems(see  \cite{BootstrapEfron}). We generate a sample from $\widehat{F}_n$ to be used as input to a simulation model(see  \cite{simulationBanks}). The first, in Algorithm \ref{algol:bootstrapVariable}, generating a sample of empiricial distribution function by following $X_1^* , \ldots , X_n^* \sim \widehat{F}_n$ of landscape random variables. In Algorithm \ref{algol:RandomVariate}, arrange the data from the smallest to the largest with the common sorting algorithm, then assign the probability $\dfrac{1}{n}$ to each interval $y_{(i-1)} < y \leq y_{(i)}$. The slope of the ith segment is given by:
\[
a_i = \dfrac{x_{(i)} - x_{(i-1)}}{1/n - (i-1)/n}.
\] 
The inverse transform technique can be used for a variety of distribution specially empirical distribution. to obtain samples, the following are performed:
\begin{enumerate}
\item Compute the CDF of the desired random variable $X$.
\item Solve the equation $F(X) = R$ for $X$ in term of $R$.
\item Generate (as needed) uniform random number $R_1,R_2,\ldots$,so on, and computed the desired random variates by:
\[
X_i = F^{-1} (R_i).
\]
\end{enumerate} 
Using the second step of Algorithm \ref{algol:bootstrapVariable}, applying the logarithm function(each derivative function) on summation of random variate, we have $T^\ast_{n,1},\ldots,T^\ast_{n,B}$.  
\begin{algorithm}
\KwData{random variables $Y= [y_1 , \ldots , y_n]$; Empiricial distribution function$\widehat{F}_n$; $k$ is number of generating sample from $\widehat{F}_n$.}
\KwResult{create vectorA = ($T^\ast_{n,1}, \ldots , T^\ast_{n,B}$).}
\Begin{
Create vectorA with dimension B $\times$ 1\;
Create vectorB with dimension k $\times$ 1\;

\For{$i \leftarrow 1$ \KwTo $B$}{
$\mbox{vectorB} \longleftarrow \mbox{call algorithm \ref{algol:RandomVariate} with input(Y,k)}$\;
$temp \longleftarrow \sum_{i=1}^n vectorB[i] $ \;
$\mbox{vectorA}[i] \longleftarrow \log(temp) $\;
}
}
\caption{Construct variables with bootstrap approach \label{algol:bootstrapVariable}}
\end{algorithm}
\begin{algorithm}
\KwData{sort variables $Y = [y_1 , \ldots , y_n]$; n which is number of sampling from $\widehat{F}_n$;}
\KwResult{vectorB which is n random variables.}
\Begin{
Create vectorB with dimension n $\times$ 1\;
R is random number with uniform distribution\;
\For{$k \leftarrow 1$ \KwTo $n$}{
\For{$i \leftarrow 1$ \KwTo $length(\mbox{Y})$}{
$CDF[i] \longleftarrow i/length(Y)$\;
}
\For{$i \leftarrow 1$ \KwTo $length(\mbox{Y})$}{
\If{$R < CDF[1]$}{
$\mbox{generateX} \longleftarrow \mbox{Y[1]} + (\mbox{Y[1]} / (1/length(\mbox{Y}))) * R $
}
\If{$i < length(\mbox{Y}) \ \mbox{and} \ (i-1 > 0) \ \mbox{and} \ R >= CDF[i] \ \mbox{and} \ R <= CDF[i+1]$}{
 $\mbox{generateX} \longleftarrow \mbox{Y}[i-1] + ((\mbox{Y}[i] - \mbox{Y}[i-1])/(1/length(\mbox{Y})))  \times \newline (R - (i-1)/length(\mbox{Y}))$
}
}
$\mbox{vectorB}[k] \longleftarrow \mbox{generateX}$
}
}
\caption{Random variate generation \label{algol:RandomVariate}}
\end{algorithm}
By the law of large numbers, in algorithm \ref{algol:vBoot}, $v_{boot} \xrightarrow[]{\text{a.s}} V_{\widehat{F}_n}(T_n)$ as $b \rightarrow \infty$.
\begin{algorithm}
\KwData{$T^\ast_n$ bootstrap sample variable.}
\KwResult{compute variance of bootstrap method.}
\Begin{
$B \longleftarrow length(T^\ast_n)$ \;
$V_{boot} \longleftarrow \dfrac{1}{B} \sum_{b=1}^B \Big( T^\ast_n[b] - \dfrac{1}{B} \sum_{r=1}^B T^\ast_n[r] \Big)^2 $\;
}
\caption{Compute  variance of bootstrap method \label{algol:vBoot}}
\end{algorithm}
There are several ways to construct a bootstrap confidence interval. In Algorithm \ref{algol:StudentizedConfidence}, the sample quantiles of the bootstrap quantities $Z^\ast_{n,1},\ldots,Z^\ast_{n,B}$ should approximate the true quantiles of the distribution of $Z_n = \dfrac{T_n - \theta}{\widehat{se}_{boot}}$. Let $z^\ast_\alpha$ denote the $\alpha$ sample quantile of $Z^\ast_{n,1},\ldots,Z^\ast_{n,B}$, then $\mathbb{P}(Z_n \leq z^\ast_\alpha) \approx \alpha$. Let
\[
C_n = \Big( T_n - z^\ast_{1- \alpha/2} \widehat{se}_{boot}, T_n - z^\ast_{\alpha/2} \widehat{se}_{boot} \Big)
\]
then, \\
$\mathbb{P}(\theta \in C_n)$ $\begin{array}[t]{l}= \mathbb{P} \Big( T_n - z^\ast_{1- \alpha/2} \widehat{se}_{boot} \leq \theta \leq T_n - z^\ast_{\alpha/2} \widehat{se}_{boot} \Big) \\
= \mathbb{P} \Big( z^\ast_{\alpha/2} \leq \dfrac{T_n - \theta}{se_{boot}} \leq z^\ast_{1- \alpha/2}\Big) \\
= \mathbb{P} \Big( z^\ast_{\alpha/2} \leq Z_n \leq z^\ast_{1 - \alpha/2} \Big) \\
\approx 1- \alpha.
\end{array}$
\begin{algorithm}
\KwData{$\alpha$; sort variables $Y = [y_1 , \ldots , y_n]$; vectorA = ($T^\ast_{n,1}, \ldots , T^\ast_{n,B}$).}
\KwResult{bootstrap Studentized confidence interval.}
\Begin{
$a \longleftarrow 1 - ( \alpha / 2)$ \;
$ b \longleftarrow alpha/2$ \;
create $Z_{n}^\ast$ with dimension of $B \times 1$ \;
$\widehat{se}_{boot} \longleftarrow \mbox{call algorithm \ref{algol:vBoot} with input vectorA}$ \;
$ T_n \longleftarrow log \Big( \sum_{i=1}^n Y[i] \Big)$\;
\For{$i \leftarrow 1$ \KwTo $B$}{
$Z_{n}[i]^\ast \longleftarrow (\mbox{vectorA}[i] - T_n ) / ((\mbox{vectorA}[i] - ( \mbox{vectorA}[i]/length(\mbox{vectorA}))^2)/length(\mbox{vectorA}))$
}
$\mbox{confidence}^+ \longleftarrow  T_n - ($ compute Quantile with  percent b on data $Z_{n,b}^\ast \times \sqrt{\widehat{se}_{boot}}$)\;
$\mbox{confidence}^- \longleftarrow  T_n - ($ compute Quantile with  percent a on data $Z_{n,b}^\ast \times \sqrt{\widehat{se}_{boot}}$) \;
}
\caption{Compute  bootstrap Studentized confidence interval \label{algol:StudentizedConfidence}}
\end{algorithm}
\newline In Algorithm \ref{algol:DeltaMethod}, we compute a large sample confidence interval is $T(\widehat{F}_n) \pm z_{\alpha/2} \widehat{se}$. In Algorithm \ref{algol:SeHat}, $\widehat{l}(x)$ is the empiricial infulence function that is equivalent Theorem \ref{eq:influenceFunctions}. 
\begin{algorithm}
\KwData{random variables $Y = [y_1 , \ldots , y_n]$; $z_{\alpha/2}$.}
\KwResult{Compute confidence interval with delta method.}
\Begin{
$n \longleftarrow length(Y)$ \;
$T(\widehat{F}_n) \longleftarrow \log(\sum_{i=1}^n Y[i])$ \;
$\mbox{confidence}^+ \longleftarrow T(\widehat{F}_n)  + ( z_{\alpha/2} \times \mbox{algorithm \ref{algol:SeHat} with input Y} )$ \;
$\mbox{confidence}^- \longleftarrow T(\widehat{F}_n) + ( z_{\alpha/2} \times \mbox{algorithm \ref{algol:SeHat} with input Y} )$ \;
}
\caption{Compute Delta Method confidence interval \label{algol:DeltaMethod}}
\end{algorithm}
\begin{algorithm}
\KwData{random variables $Y=[y_1 , \ldots , y_n]$.}
\KwResult{compute $\widehat{se}$.}
\Begin{
$temp \longleftarrow \log(\sum_{i=1}^n Y[i])$ \;
$ \widehat{l}(x) \longleftarrow matrixY - temp$ \;
$\widehat{\tau}^2 \longleftarrow \Big( \sum_{i=1}^{n} \widehat{l}(x)^2 \Big) / n$ \;
$\widehat{se} \longleftarrow \sqrt{\widehat{\tau}^2} / \sqrt{n}$ \;
}
\caption{Compute  $\widehat{se}$ \label{algol:SeHat}}
\end{algorithm}
\subsubsection{Density Estimation Persistence Landscape}
Let $Y= (y_1,\ldots,y_n)$ and $x^\ast = y_i$, we compute $B_{x^\ast} = \lbrace y \ | \ |y - x^\ast| < h \rbrace$ for $x^\ast$ and replace $B_{x^\ast}$ with $X_i$ in Definition \ref{def:densityEstimator}. In cross validation (Definition \ref{eq:crossValidation}), we return $h$ which is the minimum square error loss function. We choose minimum $h$, which is the optimal cross-validation estimator of risk (Definition \ref{eq:crossValidation}). Now, we apply Algorithm \ref{algol:Risk} for all of the random variables generated by Algorithm \ref{algol:RandomVariate} and then obtained theorem \ref{eq:theorem612} for density estimator of persistence landscapes.
\begin{algorithm}
\KwData{random variables $Y = [y_1 , \ldots , y_n]$; bandwidth $h$, guassian kernel $K$.}
\KwResult{compute matrix $\widehat{f}(x^\ast)$.}
\Begin{
create $\widehat{f}(x)$ with dimension $n \times 1$ \;
\For{$i  \leftarrow 1$ \KwTo $length(matrixY)$}{
$x^\ast \longleftarrow Y[i]$ \;
$B[x^\ast] \longleftarrow \lbrace x \ | \ |Y[i] - x^\ast| < h \rbrace$ \;
$\widehat{f}(x^\ast) = \dfrac{1}{n} \sum_{i=1}^n \dfrac{1}{n} K \Big( \dfrac{x^\ast - B[x^\ast]}{h} \Big).$
}
}
\caption{Compute  density estimator \label{algol:densityEstimator}}
\end{algorithm}
\begin{algorithm} \label{algol:Risk}
\KwData{random variable $x$, $h^\ast$.}
\KwResult{compute $f^\prime$.}
\Begin{
compute cluster $x_0$ of bootstrap random variable $x$ with distance $|x-x_0| < h^\ast $ \;
$ f^\prime(x,x_0^i) \longleftarrow \dfrac{\widehat{f}(x) - \widehat{f}(x_0^i)}{x-x_0^i}$ \;
$f^\prime(x,x_0) \longleftarrow \dfrac{\sum f^\prime(x,x_0^i)}{length(x_0)}$ \;
}
\caption{Compute integrated mean square error}
\end{algorithm}
\section{Applications} \label{section:Applications}
In this section, we calculated the nonparametric methods on persistence landscapes to confirm accuracy of our methods respect to another approach, using \textsf{R}  programming language with TDA package by \cite{TDARFasy}.
\subsection{Sphere and Torus}
\cite{DiaconisSampling} developed an algorithm for sampling submanifold with a probability distribution. In this section, we sample from the sphere and torus  uniformly with respect to the surface.
Let $R$  be the major radius and  $r$ as the minor radius, we use an explicit equation in Cartesian coordinates for a torus, which is:
\[
\Big( R - \sqrt{x^2 + y^2} \Big)^2 + z^2 = r^2.
\]
For $1000$ points, we construct a filtered simplicial complex as follows. First, we form the Vietoris-Rips complex $R(X,\epsilon)$, which consists of simplices with vertices in $X = \lbrace x_1,\ldots,x_n \rbrace \subset \mathbb{R}^d$ and diameter at most $\epsilon$. The sequence of Vietoris-Rips complex obtained by gradually increasing the radius $\epsilon$  create a filtration of complexes. We denote the limit of filtration of the Vietoris-Rips complex with $5$ and maximum dimension of homological feature with $1$($0$ for components, $1$ for loops). To compute landscape function in Equation \ref{eq:landscapeBubenik}, we set $t \in [0,5],k = 1$. We construct $100$ random variables by Equation \ref{functionalBubenik}, the logarithm function is our plug-in estimator, and the empirical influence function is different among random variables with plug-in estimator. As can be seen from Figure \ref{fig:sphereTorusDelta}, we repeated the Algorithm \ref{algol:DeltaMethod} for $100$ times to obtain the upper and lower confidence interval. Table \ref{tbl:bootstrap1} present the nonparametric bootstrap computed using the approach for a  $95\%$ critical value with a few assumptions about persistence landsapces. As shown in Figures \ref{fig:sphereDensity} and  \ref{fig:torusDensity}, we create $100$ random variables and $500$ times bootstrap sample data (see Algorithm \ref{algol:RandomVariate} ) and replaced with orginal data. We showed that using a confidence interval such as $\bar{Y} \pm z^\ast \dfrac{\sigma}{\sqrt{n}}$, gives $0.06628939$ for density estimation of the sphere and $0.02067551$ for torus, which is difference between upper and lower confidence interval. On the other hand, using nonparametric method with correct kernel as the tricube kernel and $h^\ast$, we obtained $0.0004472946$ for sphere and $0.0003435891$ for torus points, which are significant different. Now, to evaluate the quality of an estimator $\widehat{f}_n$ with respect to $f$ with integrated mean squared error, we apply Algorithm \ref{algol:Risk} which is obtain Figure \ref{fig:RiskofSphere} for $100$ times with $0.002$ precision of bandwidth $h$ and Gaussian kernel for sphere points and  for torus with difference between below and upper confidence interval in $100$ times, is $0.00004416$. 
\begin{figure}
\centering
\includegraphics[scale=0.45]{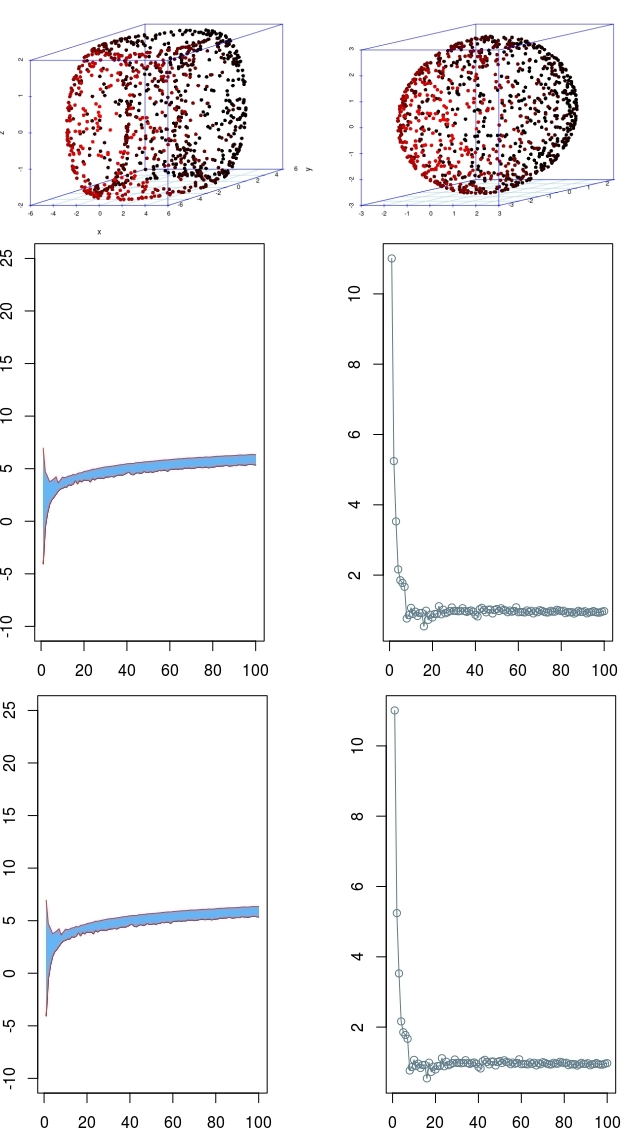}
\caption{ In row $1$, using $1000 \times 100$ simulated data from uniform distribution ,$100$ times for each point, for torus(column $1$) and sphere(column $2$), have been computed the random variables of persistence landscape by delta method. In row $2$, column $1$, for sphere and row $3$, column $1$, the torus are shown. row $2$, column $2$ shows the difference between the upper and lower confidence interval of delta method of the sphere.Similarly, row $3$, column $2$ for the torus is shown.
}
\label{fig:sphereTorusDelta}
\end{figure}
\begin{figure}
\centering
\includegraphics[scale=0.62]{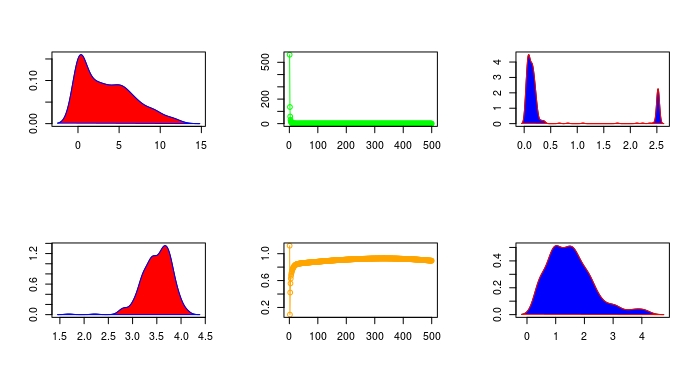}
\caption{In row $1$,column $1$, using $500$ simulated data from uniform distribution for a sphere with radius $2$, the density of random variables of persistence landscape have been drawn. In row $1$, column $2$, $\widehat{j}(h)$ with percision $0.002$  that minimum value is $0.0029$ have been drawn. In row $1$, column $3$, kernel density estimator with bandwidth $0.056$ have been drawn. In row $2$, column $1$ using the bootstrap method for alternate generating random variate with persistence landscape have been drawn. In row $2$, column $2$ value of $\widehat{j}(h)$ with percision $0.002$  that minimum value  $0.0934$ is drawn. In row $2$, column $3$ plot kernel density estimator with bandwidth $0.004$ is drawn.}
\label{fig:sphereDensity}
\end{figure}
\begin{figure}
\centering
\includegraphics[scale=0.62]{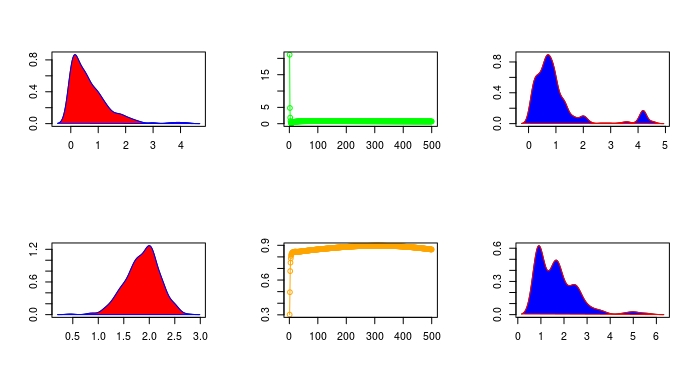}
\caption{In row $1$, column $1$, using  $500$ simulated data from uniform distribution for a torus with $R = 2$ and $r =1$ , the density of random variables of persistence landscape have been drawn. In row $1$, column $2$, $\widehat{j}(h)$ with percision $0.002$  that minimum value is $0.0236$ have been drawn. In row $1$, column $3$, kernel density estimator with bandwidth $0.014$ have been drawn. In row $2$, column $1$ using the bootstrap method for alternative generating random variate with persistence landscape have been drawn. In row $2$, column $2$ value of $\widehat{j}(h)$ with percision $0.002$ that minimum value $0.3019$ is drawn. In row $2$, column $3$, the kernel density estimator with bandwidth $0.004$ is drawn.}
\label{fig:torusDensity}
\end{figure}
\begin{figure}
\centering
\includegraphics[scale=0.48]{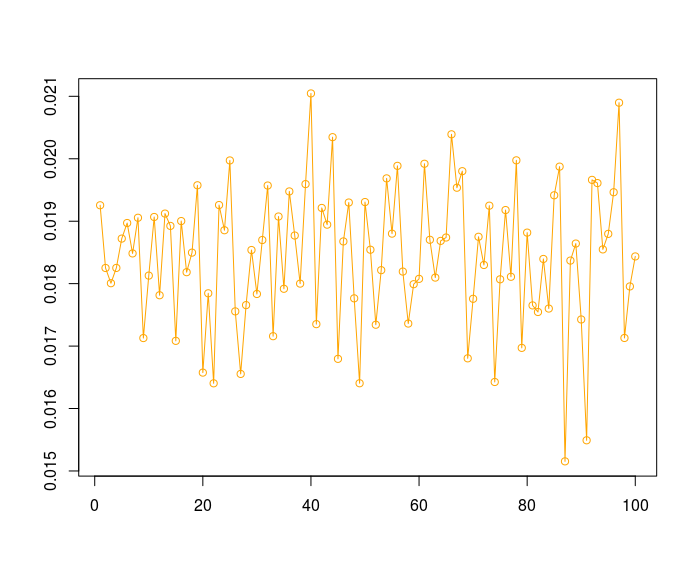}
\includegraphics[scale=0.48]{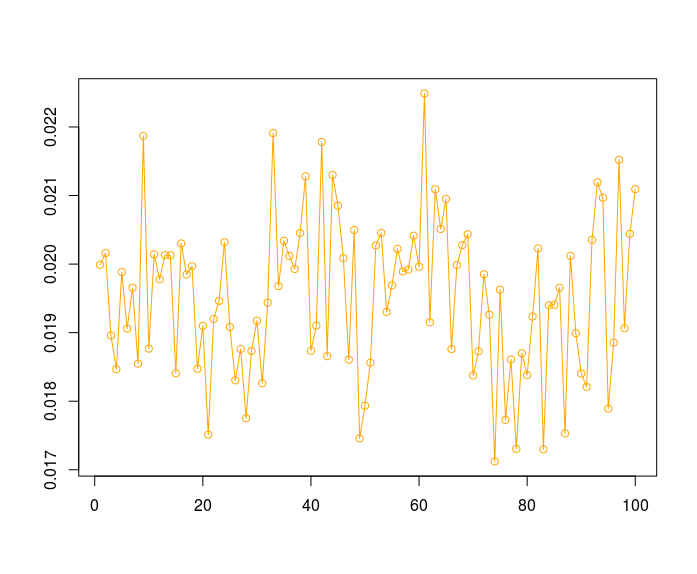}
\caption{We run $100$ times to evaluate minimize estimated risk $R(\widehat{f}_n,f)$ for points on sphere in row $1$ and torus in row $2$ with Guassian kernel.}
\label{fig:RiskofSphere}
\end{figure}
\begin{table}
\begin{center}
\begin{tabular}{ *2l }    \toprule
\emph{Method} & \emph{$95\%$ Interval}  \\\midrule
pivotal    & $(4.248799,4.319345)$  \\ 
normal & $(3.929793,4.406627)$ \\ 
studentize & $(4.297297,4.418873)$ \\
percentile & $(4.017076,4.087621)$ \\\bottomrule
 \hline
\end{tabular}
\quad
\begin{tabular}{ *2l }    \toprule
\emph{Method} & \emph{$95\%$ Interval}  \\\midrule
pivotal    & $(1.498743,1.540386)$  \\ 
normal & $(1.241105,1.566729)$ \\ 
studentize & $(2.415224,2.95656)$ \\
percentile & $(1.267447,1.309091)$ \\\bottomrule
 \hline
\end{tabular}
\end{center}
\caption{In the left column, we calculated bootstrap confidence interval with four commonly used accurate approaches. We sampled  $1000$ points for the sphere. The right column is the same for torus points.}
\label{tbl:bootstrap1}
\end{table}
\subsection{Breast Cancer}
Considering the application of computational topology on some dataset, extract topological invariant of real data set such as prognosis and diagnosis of breast cancer is one of the important in biological research. This dataset (now is available in UCI machine learning repository) consists of radius, perimeter, area, compactness and another attribute for each cell nucleus. Features are computed from a digitized image of a fine needle aspirate (FNA) of a breast mass. They describe characteristics of the cell nuclei present in the image. Some related publications on this subject can be found in \cite{InstituteBreastCancer}. For applying our approach on this dataset, we sampled $500$ points from all of data and then constructed persistence landscapes based on Definition \ref{eq:landscapeBubenik}: $t \in [0,5],k = 1$. In figure \ref{fig:cancerDensity}, if we compute difference between upper and lower confidence interval for random variables of persistence landscape, we obtain a value of $0.001335084$ but applying a density estimation for random variables gives a precision value of $0.0009229557$ with $h=0.006$ and minimum of $\widehat{J}(h) = 0.1318038 $ in equation \ref{eq:crossValidation}. From Figure \ref{fig:deltaCancer}, we obtained $0.01235495$ and $0.01237665$ for a integrated mean squared error as lower and upper confidence interval, respectively.
\begin{table}
\begin{center}
\begin{tabular}{ *2l }    \toprule
\emph{Method} & \emph{$95\%$ Interval}  \\\midrule
pivotal    & $(1.607248,1.641904)$  \\ 
normal & $(1.356374,1.752889)$ \\ 
studentize & $(2.073679,2.456706)$ \\
percentile & $(1.46736,1.502016)$ \\\bottomrule
 \hline
\end{tabular}
\end{center}
\caption{In the column, we calculate bootstrap confidence interval with four commonly used precise approaches. We sampled $500$ points of breast cancer dataset.}
\label{tbl:bootstrap2}
\end{table}
\begin{figure}
\centering
\frame{\includegraphics[scale=0.4]{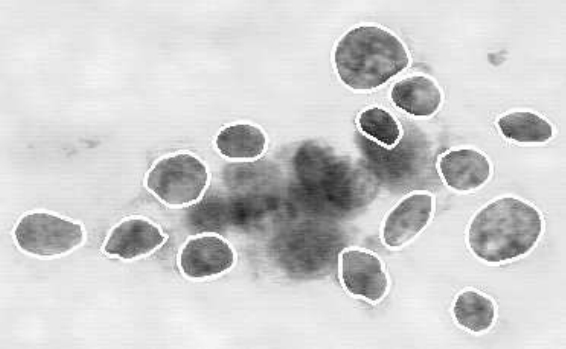}}
\includegraphics[scale=0.62]{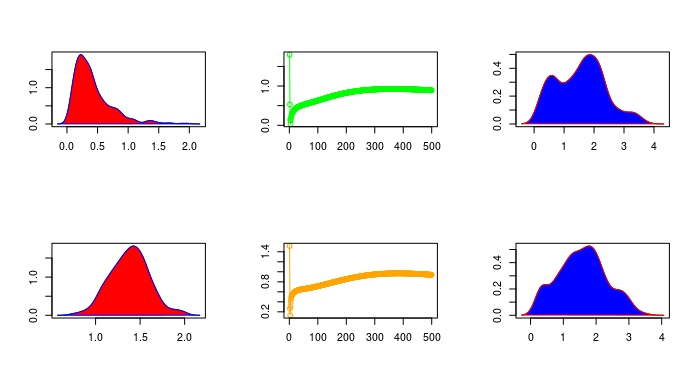}
\caption{In row 1, it is a magnified image of a malignant breast FNA. \cite{MangasarianBreastCancer}. In row $2$ and column $1$, from uniform distribution, we sampled $500$ points for a breast cancer dataset which we have ploted density of random variables of persistence landscape, in row $2$, column $2$ we plotted $\widehat{j}(h)$ with percision $0.002$  that minimum value is $0.03028$, in row $2$, column $3$ plot kernel density estimator with bandwidth $0.01$ tricube kernel. In row $3$, column $1$ we use bootstrap method for alternate generating random variate with persistence landscape, in row $3$, column $2$ we plotted $\widehat{j}(h)$ with percision $0.002$  that minimum value is $0.02581$, and finally, in row $3$, column $3$ we plotted kernel density estimator with bandwidth $0.006$ and tricube kernel.}
\label{fig:cancerDensity}
\end{figure}
\begin{figure}
\centering
\includegraphics[scale=0.62]{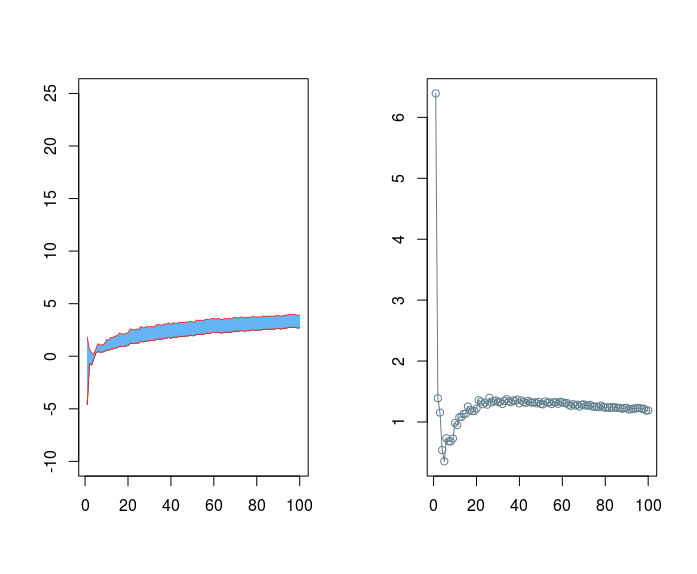}
\includegraphics[scale=0.48]{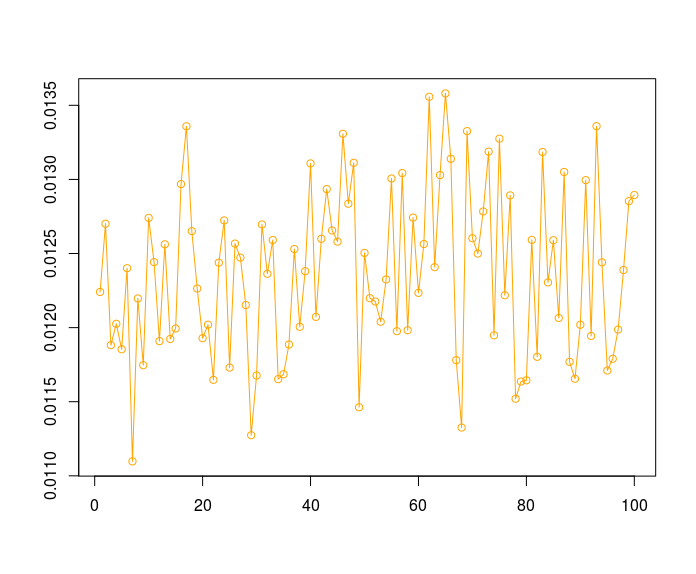}
\caption{In column $1$ for dataset, we sampled $500$ points of breast cancer dataset, for each $100$ times each, and we have computed delta method for random variables of persistence landscape. In column $2$  difference between upper and lower of the confidence interval of delta method is presented. In row $2$, we have presented minimum risk estimator $R(\widehat{f}_n,f)$ for points on breast cancer dataset with Guassian kernel.}
\label{fig:deltaCancer}
\end{figure}
\subsection*{Acknowledgments}
The authors gratefully acknowledge the support of the center of statistical learning and its application at Allameh Tabatabai University (under grant No. P/H/040). We would like to thank Naiereh Elyasi for her helpful discussions.
\newpage
\appendix
\section*{Appendix A.}
\begin{proof} \label{proof:611} of Theorem \ref{eq:theorem612}. For any $x,u \in B_j$,
\[
f(u) = f(x) + (u+x)f^\prime(x) + \dfrac{(u-x)^2}{2}f^{''}(\tilde{x}) 
\]
for some $\tilde{x}$ between $x$ and $u$. Hence, 
\begin{align*}
p_j = {\displaystyle \int_{B_{j}}} f(u) du \begin{array}[t]{l}= {\displaystyle \int_{B_{j}}} \Big( f(x) + (u-x)f^\prime(x) + \dfrac{(u-x)^2}{2} f^{''}(\tilde{x}) \Big) du \\ = f(x)h + hf^\prime(x) \Big(h \Big( j - \dfrac{1}{2} \Big) - x \Big) + O(h^3).
\end{array}
\end{align*}
Therefore, the bias $b(x)$ is 
\begin{align*}
b(x) \begin{array}[t]{l}= \mathbb{E}(\widehat{f}_n(x)) - f(x) = \dfrac{p_j}{h} - f(x) \\ = \dfrac{f(x)h + hf^\prime (h(j - \dfrac{1}{2}) - x) + O(h^3)}{h} - f(x) \\ = f^\prime (x) \Big( h \Big( j - \dfrac{1}{2} \Big) - x \Big) + O (h^2). \end{array}
\end{align*}
By the mean value theorem we have, for some $\tilde{x}_j \in B_j$, that
\begin{align*}
{\displaystyle \int_{B_{j}}} b^2(x)dx \begin{array}[t]{l} = {\displaystyle \int_{B_{j}}} (f^\prime(x))^2 \Big( h(j - \dfrac{1}{2}) - x \Big)^2 dx + O(h^4) \\ = (f^\prime (\tilde{x_j}))^2 {\displaystyle \int_{B_{j}}} \Big( h(j - \dfrac{1}{2}) - x \Big)^2 dx + O(h^4) \\ = (f^\prime (\tilde{x_j}))^2 \dfrac{h^3}{12} + O(h^4). \end{array}
\end{align*}
Therefore, 
\begin{align*}
{\displaystyle \int_{0}^1} b^2(x) dx \begin{array}[t]{l} = {\displaystyle \sum_{j=1}^m \int_{B_{j}}} b^2(x) dx + O(h^3) \\ = {\displaystyle \sum_{j=1}^m } (f^\prime(\tilde{x_j}))^2 \dfrac{h^3}{12} + O(h^3) \\ = \dfrac{h^2}{12} {\displaystyle \sum_{j=1}^m } h(f^\prime(\tilde{x_j}))^2 + O(h^3) \\ = \dfrac{h^2}{12} {\displaystyle \int_0^1} (f^\prime (x))^2 dx + o(h^2).  \end{array}
\end{align*}
Now consider the variance. By the mean value theorem, $p_j = {\displaystyle \int_{B_{j}}} f(x) dx = hf(x_j)$ for some $x_j \in B_j$. Hence, with $v(x) = \mathbb{V}(\widehat{f}_n(x))$, 
\begin{align*}
{\displaystyle \int_0^1} v(x) dx \begin{array}[t]{l} = {\displaystyle \sum_{j=1}^m} {\displaystyle \int_{B_{j}}} v(x) dx = {\displaystyle \sum_{j=1}^m} {\displaystyle \int_{B_{j}}} \dfrac{p_j (1 - p_j)}{nh^2} \\ = \dfrac{1}{nh} {\displaystyle \sum_{j=1}^m \int_{B_{j}}} p_j - \dfrac{1}{nh^2} {\displaystyle \sum_{j=1}^m \int_{B_{j}}} p_j^2 \\ = \dfrac{1}{nh} - \dfrac{1}{nh}{\displaystyle \sum_{j=1}^m} h^2 f^2(x_j) = \dfrac{1}{nh} - \dfrac{1}{n} \sum_{j=1}^m hf^2(x_j) \\ = \dfrac{1}{nh} - \dfrac{1}{n} \Big({\displaystyle \int_{0}^1} f^2(x)dx + o(1) \Big) = \dfrac{1}{nh} + o (\dfrac{1}{n}). \end{array}
\end{align*}
\end{proof}
\begin{proof}\label{proof:628} the Theorem \ref{theorem:628}.
Write $K_h(x,X) = h^{-1} K \Big( (x-X)/h \Big)$ and $\widehat{f}_n(x) = n^{-1} \sum_i K_h (x,X_i)$. Thus $\mathbb{E} [\widehat{f}_n(x)] = \mathbb{E} [K_h (x,X)]$ and $\mathbb{V}[\widehat{f}_n(x)] = n^{-1} \mathbb{V} [K_h(x,X)]$. Now,
\begin{align*}
\mathbb{E} [K_h (x,X)]  \begin{array}[t]{l} = \int \dfrac{1}{h} K \Big( \dfrac{x-t}{h} \Big) f(t) dt \\
= \int K(u) f(x-hu) du \\
= \int K(u) \Big[ f(x) - huf^\prime(x) + \dfrac{h^2 u^2}{2} f^{''}(x) + \ldots \Big] du \\
= f(x) + \dfrac{1}{2} h^2 f^{''}(x) \int u^2 K(u) du \ldots \end{array}
\end{align*}
since $\int K(x) dx = 1$ and $\int x K(x) dx = 0$. The bias is
\begin{align*}
\mathbb{E} [K_{h_{n}} (x,X)] - f(x) = \dfrac{1}{2} \sigma_K^2 h_n^2 f^{''}(x) + O(h_n^4).
\end{align*}
By a similar calculation,
\begin{align*}
\mathbb{V} [\widehat{f}_n(x)] = \dfrac{f(x) \int K^2 (x) dx}{nh_n} + O \Big( \dfrac{1}{n} \Big).
\end{align*}
\end{proof}

\end{document}